 \documentclass[accepted]{uai2021} 

\usepackage[american]{babel}

\usepackage{natbib} 
\usepackage{mathtools} 
\usepackage{booktabs} 
\usepackage{tikz} 

\usepackage[ruled, vlined, linesnumbered]{algorithm2e}

\usepackage{amsmath}
\usepackage{mathtools}
\usepackage{amsthm}
\usepackage{amsfonts}

\usepackage{scalerel,stackengine}
\newtheorem{lem}{Lemma}

\theoremstyle{definition} \newtheorem{defn}{Definition}
\theoremstyle{definition} \newtheorem{assumption}{Assumption}
\theoremstyle{definition} \newtheorem{exm}{Example}
\newcommand{\indepe}{\mathop{\perp\!\!\!\perp}}
\newcommand{\notindepe}{\mathop{\perp\!\!\!\!\!\!/\!\!\!\!\!\!\perp}}

\stackMath
\newcommand\reallywidehat[1]{%
\savestack{\tmpbox}{\stretchto{%
  \scaleto{%
    \scalerel*[\widthof{\ensuremath{#1}}]{\kern-.6pt\bigwedge\kern-.6pt}%
    {\rule[-\textheight/2]{1ex}{\textheight}}
  }{\textheight}%
}{0.5ex}}%
\stackon[1pt]{#1}{\tmpbox}%
}
\def\rightarrowCirc{\hbox{$\circ$}\kern-1.5pt\hbox{$\rightarrow$}}
\def\circHyphenCirc{\hbox{$\circ$}\kern-1.5pt\hbox{$-$}\kern-1.5pt\hbox{$\circ$}}
\def\circHyphen{\hbox{$\circ$}\kern-1.5pt\hbox{$-$}}

\newcommand{\argmax}{\mathop{\rm argmax}\limits}

\makeatletter
\newcommand*\bigcdot{\mathpalette\bigcdot@{.5}}
\newcommand*\bigcdot@[2]{\mathbin{\vcenter{\hbox{\scalebox{#2}{$\m@th#1\bullet$}}}}}
\makeatother

\title{Discovery of Causal Additive Models in the Presence of Unobserved Variables}

\author[1]{Takashi Nicholas Maeda} 
\author[1,2]{Shohei Shimizu}
\affil[1]{%
    RIKEN Center for Advanced Intelligence Project
}
\affil[2]{%
    Shiga University
}

\begin{document}
\maketitle

\begin{abstract}
Causal discovery from data affected by unobserved variables is an important but difficult problem to solve. The effects that unobserved variables have on the relationships between observed variables are more complex in nonlinear cases than in linear cases. In this study, we focus on causal additive models in the presence of unobserved variables. Causal additive models exhibit structural equations that are additive in the variables and error terms. We take into account the presence of not only unobserved common causes but also unobserved intermediate variables. Our theoretical results show that, when the causal relationships are nonlinear and there are unobserved variables, it is not possible to identify all the causal relationships between observed variables through regression and independence tests. However, our theoretical results also show that it is possible to avoid incorrect inferences. We propose a method to identify all the causal relationships that are theoretically possible to identify without being biased by unobserved variables. The empirical results using artificial data and simulated functional magnetic resonance imaging (fMRI) data show that our method effectively infers causal structures in the presence of unobserved variables.
\end{abstract}

\section{Introduction}
\label{sec:intro}

\begin{figure*}[t]
\centering
\includegraphics[width=16.0cm]{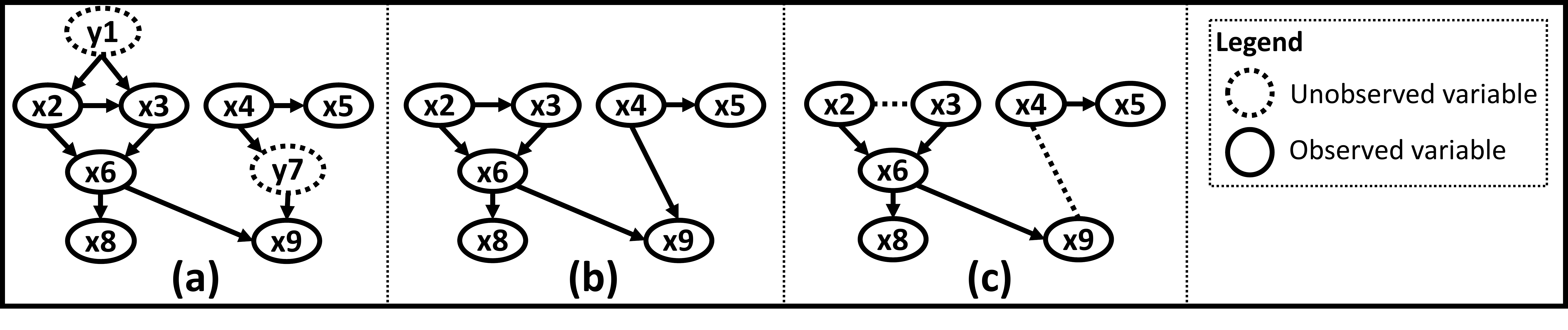}
\caption{(a) Data generation process. (b) True causal graph of observed variables. (c) Causal graph inferred by our method: Each dashed edge indicates that the causal relationship cannot be determined by our proposed method.}
\label{figure:intro}
\end{figure*}

A fundamental objective in various fields of science is to identify causal relationships. While randomized control trials are the most effective means of understanding causal relationships, such an approach is often too costly, unethical, or technically impossible to conduct. Thus, causal discovery from purely observational data is very important for scientific research.\par
Causal discovery methods often assume that the causal structures form directed acyclic graphs (DAGs) and that unobserved common causes are absent~\citep{Spirtes91,shimizu2006,shimizu2011,hoyer2009,Mooij:2009,peters2014}. If methods that assume the absence of unobserved variables are applied to data affected by unobserved variables, the causal graphs inferred by such methods are biased, and thus tend to be incorrect. The fast causal inference (FCI)~\citep{fci} and RFCI~\citep{colombo2012} both assume the presence of unobserved common causes and can present variable pairs with unobserved common causes. However, they infer causal relationships based on conditional independence, and thus cannot distinguish between causal graphs that entail the same sets of conditional independence.

Until recently, causal functional model-based approaches~\citep{shimizu2011,hoyer2009,Mooij:2009,zhang2009identifiability,petIdentifiability,peters2014} had not been used to explore causal models with unobserved variables. Causal functional model-based approaches assume that causal effects can be formulated with a specific form of functions. For example, LiNGAM~\citep{shimizu2006,shimizu2011} and additive noise models (ANMs)~\citep{hoyer2009} assume that the data generation process can be formulated as $x_i = f_i({\bf pa}_i)+n_i$, where $x_i$ is an observed variable, ${\bf pa}_i$ is the set of the direct causes ({\it parents}) of $x_i$, and $n_i$ is the external effect on $x_i$. These methods identify the causal direction between observed variables $x_i$ and $x_j$ as $x_j \rightarrow x_i$ if the residual of $x_i$ regressed on $x_j$ is independent of $x_j$ and the residual of $x_j$ regressed on $x_i$ is dependent of $x_i$. When analyzing data suited to the models, these approaches can identify the entire causal model.

Recently, a causal functional model-based method called repetitive causal discovery (RCD)~\citep{maeda20a}, an extension of DirectLiNGAM~\citep{shimizu2011}, was proposed. RCD infers causal graphs in which bi-directed edges represent variable pairs affected by unobserved common causes and directed edges represent the direct causal relationships between observed variables. The RCD method assumes that the causal relationships are linear and the external effects are non-Gaussian. It infers that $x_i$ and $x_j$ have unobserved common causes when the residual of $x_i$ regressed on $x_j$ is dependent of $x_j$, and vice versa. \citet{Janzing2009} proposed a method to identify the causal relationships between a pair of observed variables assuming that there exists a single unobserved common cause and that the causal functions are nonlinear. However, little research has been conducted on the discovery of causal structures with three or more observed variables, assuming nonlinear causal relationships and the presence of unobserved variables.\par
The effects that unobserved variables have on the relationships between observed variables are more complex in nonlinear cases than in linear cases. Assume that the causal effect from $x_j$ to $x_i$ is indirectly mediated through unobserved variable $y_k$ (i.e., $x_j \rightarrow y_k \rightarrow x_i$). Then, $(\exists f,\ [x_i - f(x_j) \indepe x_j])$ holds in linear cases but it does not hold in nonlinear cases (i.e., $(\forall f,\ [x_i - f(x_j) \notindepe x_j])$). Therefore, the causal relationship between $x_i$ and $x_j$ cannot be determined by regression methods. This is called a cascade ANM (CANM) and has been intensively discussed by \citet{caicausal}. However, their proposed method assumes that there is no unobserved common cause. Therefore, it is not applicable for inferring causal relationships between three or more observed variables.

Our study is aimed at extending causal additive models (CAMs)~\citep{buhlmann2014} to incorporate unobserved variables. CAMs are special cases of ANMs, and they assume that the structural equations are additive in the variables and error terms. We call our extended models {\it causal additive models with unobserved variables} (CAM-UV). In these models, we consider the identifiability of the causal relationships between observed variables. The theoretical results show that it is not possible to identify all the causal relationships, but it is possible to avoid incorrect inferences of causal relationships. We propose a method to infer causal relationships in CAM-UV. Assume that the data generation process is as shown in Figure~\ref{figure:intro}-(a), in which $y_1$ and $y_7$ are unobserved variables and the other nodes indicate observed variables. Ideally, the causal graph shown in Figure~\ref{figure:intro}-(b) should be recovered. However, our goal is to recover the causal graph shown in Figure~\ref{figure:intro}-(c) where the dashed undirected edges between $x_2$ and $x_3$ and between $x_4$ and $x_9$ indicate that their causal relationships cannot be identified based on our theoretical results.

The contributions of our study are as follows.
\begin{itemize}
	\item We show the identifiability of the causal relationships between observed variables in {\it causal additive models with unobserved variables} (CAM-UV).
	\item We propose a method to infer the causal graph of CAM-UV. Although the method cannot identify all the causal relationships, it can avoid incorrect inferences.
 	\item We provide experimental results on our method and compare them to existing methods using artificial data generated from CAM-UV and simulated functional magnetic resonance imaging (fMRI) data.
\end{itemize}

All the proofs are available in the Supplementary materials.

\section{Model definition}
\label{sec:problem}
Let $X=\{x_i\}$ denote the set of observed variables, $Y=\{y_i\}$ the set of unobserved variables, and $V=\{v_i\}$ the set of all the observed and unobserved variables ($V=X\cup Y$). We assume the data generation model is formulated as

\begin{equation}
	v_i= z_i + w_i + n_i,\ \ \  z_i = \sum_{x_j \in P_i} f_j^{(i)}(x_j),\ \ \ w_i=\sum_{y_k \in Q_i} f_k^{(i)}(y_k)	\label{eq:obdef2}
\end{equation}
where $z_i$ is the sum of the direct effects of observed variables on $v_i$, $w_i$ is the sum of the direct effects of unobserved variables on $v_i$, $f_j^{(i)}$ is a nonlinear function, $P_i$ is the set of observed direct causes of $v_i$, $Q_i$ is the set of unobserved direct causes of $v_i$, and $n_i$ is the external effect on $v_i$. We assume that all the external effects are mutually independent. We also assume that the causal structure of the observed and unobserved variables forms a DAG.

In addition, we impose Assumption~\ref{assumption:as1} (described below) on the causal functions and the external effects in a similar way to the Faithfulness assumption~\citep{pearl2000,spirtes2000}. According to Equation~\ref{eq:obdef2}, all the observed and unobserved variables are mixtures of external effects generated by the causal functions. In addition, there is an external effect influencing two variables if the two variables have a common {\it ancestor} (direct or indirect cause) or there is a direct or indirect effect between them. In Assumption~\ref{assumption:as1}, we assume that such variables are mutually dependent.
\begin{assumption}
	\label{assumption:as1}
We assume that all the causal functions and the external effects in CAM-UV satisfy the following condition: If variables $v_i$ and $v_j$ have terms involving functions of the same external effect $n_k$ , then $v_i$ and $v_j$ are mutually dependent (i.e., $(n_k\notindepe v_i)\land (n_k\notindepe v_j)\Rightarrow (v_i \notindepe v_j ) $). 
\end{assumption}

\section{Identifiability}
\label{sec:identifiability}
In this section, we consider the identifiability of the causal relationships between observed variables in CAM-UV.

First, we provide Definitions~\ref{defn:def1} and \ref{defn:def2}, which are used in the analysis of the identifiability. The explanatory chart for the definitions is shown in Figure~\ref{figure:path}.

\begin{figure}[t]
\centering
\includegraphics[width=8.3cm]{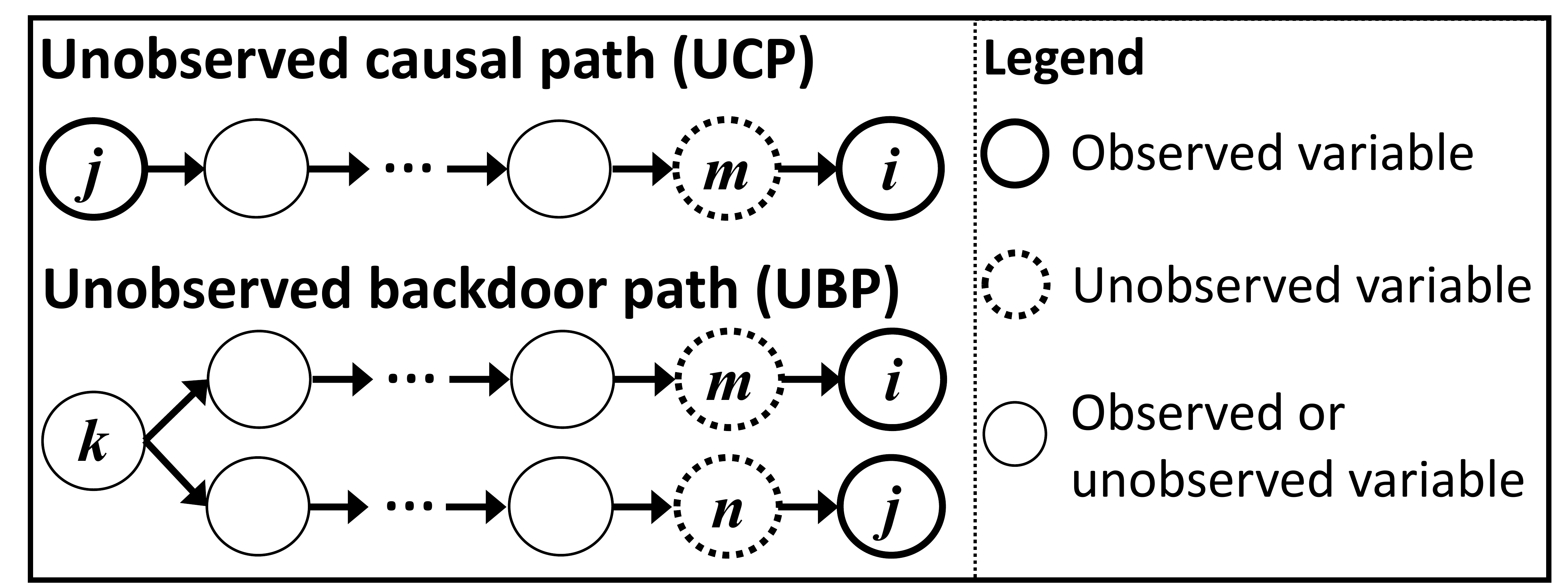}
\caption{Definitions of an unobserved causal path (UCP) and an unobserved backdoor path (UBP).}
\label{figure:path}
\end{figure}

\begin{defn}\label{defn:def1}
A directed path from an observed variable to another is called a {\it causal path} (CP). A CP from $x_j$ to $x_i$ is called an {\it unobserved causal path} (UCP) if it ends with the directed edge connecting $x_i$ and its unobserved direct cause (i.e., $x_j\rightarrow \cdots \rightarrow y_m\rightarrow x_i$ where $y_m$ is an unobserved direct cause of $x_i$).
\end{defn}

\begin{defn}\label{defn:def2}
An undirected path between $x_i$ and $x_j$ is called a {\it backdoor path} (BP) if it consists of the two directed paths from a common ancestor of $x_i$ and $x_j$ to $x_i$ and $x_j$ (i.e., $x_i\leftarrow \cdots \leftarrow v_k \rightarrow \cdots \rightarrow x_j$, where $v_k$ is the common ancestor). A BP between $x_i$ and $x_j$ is called an {\it unobserved backdoor path} (UBP) if it starts with the edge connecting $x_i$ and its unobserved direct cause, and ends with the edge connecting $x_j$ and its unobserved direct cause (i.e., $x_i\leftarrow y_m \leftarrow \cdots \leftarrow v_k \rightarrow \cdots \rightarrow y_n \rightarrow x_j$, where $v_k$ is the common ancestor and $y_m$ and $y_n$ are the unobserved direct causes of $x_i$ and $x_j$, respectively). The undirected path $x_i\leftarrow y_k \rightarrow x_j$ is also a UBP, as $v_k$, $y_m$, and $y_n$ can be the same variable.

\end{defn}

We impose Assumption~\ref{assumption:ass2} on the regression functions $G_i$ used in the lemmas provided in this section. 
\begin{assumption}\label{assumption:ass2}
Let $M$ and $N$ denote sets satisfying $M\subseteq X$ and $N\subseteq X$ where $X$ is the set of all the observed variables in CAM-UV defined in Section~\ref{sec:problem}. We assume that functions $G_i(M)$ take the forms of generalized additive models (GAMs)~\citep{hastie1990generalized} such that $G_i(M)=\sum_{x_m\in M}g_{i,m}(x_m)$ where each $g_{i,m}(x_m)$ is a nonlinear function of $x_m$. In addition, we assume that functions $G_i$ satisfy the following condition: When both $(x_i-G_i(M))$ and $(x_j-G_j(N))$ have terms involving functions of the same external effect $n_k$, then $(x_i-G_i(M))$ and $ (x_j-G_j(N))$ are mutually dependent  (i.e., $(n_k\notindepe x_i-G_i(M))\land (n_k\notindepe x_j-G_j(N))\Rightarrow ((x_i-G_i(M)) \notindepe (x_j-G_j(N)) ) $).
\end{assumption}

We first show the difference between linear and nonlinear cases of how UCPs and UBPs affect the identifiability of causality. If there is a UCP $x_j\rightarrow y_k\rightarrow x_i$, then $\exists a\in \mathbb{R}, [x_i-ax_j\indepe x_j]$ holds in linear cases~\citep{shimizu2011} but $\forall g, [x_i-g(x_j)\notindepe x_j]$ holds in nonlinear cases. That is, there is no regression function such that the residual of $x_i$ regressed on $x_j$ is independent of $x_j$. The observed variable $x_i$ is formulated as $x_i= f^{(i)}_k(f^{(k)}_j(x_j)+n_k)+n_i$. When $f^{(i)}_k$ is a nonlinear function, it cannot be represented as the linear sum of functions of $x_j$ and $n_k$ such as $f^{(i)}_k(f^{(k)}_j(x_j)+n_k) = s(x_j)+t(n_k)$. Therefore, $g(x_j)$ cannot cancel out terms containing $x_j$ from $x_i$ because $g(x_j)$ does not contain $n_k$. Therefore, when there is a UCP between $x_i$ and $x_j$, the causal relationship between $x_i$ and $x_j$ cannot be identified through regression and independence tests.\\
When there is a UBP, there is also a difference in the identifiability of causality between linear and nonlinear cases. In linear cases, the causal relationship between $x_i$ and $x_j$ can be identified if there is a set of observed variables $M\subseteq X\setminus\{x_i,x_j\}$ that blocks all the BPs between $x_i$ and $x_j$. That is, there exists a variable $x_l\in M$ on each BP~\citep{maeda20a}. In nonlinear cases, the causal relationship between $x_i$ and $x_j$ cannot be identified when there is a UBP, regardless of whether it is blocked by observed variables. Let $v_m$ denote the common ancestor of $x_i$ and $x_j$ on a UBP. Because the directed paths from $v_m$ to $x_i$ and to $x_j$ end with their unobserved direct causes, the effect of $v_m$ cannot be removed from $x_i$ or $x_j$ by regression. Therefore, the causal relationship between $x_i$ and $x_j$ cannot be identified when there is a UBP.

In the following, we provide lemmas about the identifiability of causal relationships in CAM-UV. Lemma~\ref{lem:lem1} is about the conditions in which the causal relationship between two observed variables cannot be identified. Lemma~\ref{lem:lem2} is about the condition in which the absence of the direct causal relationship between two observed variables can be identified. Finally, Lemma~\ref{lem:lem3} is about the condition in which the existence and direction of the direct causal relationship between two observed variables can be identified. We provide Examples~1, 2, and 3 for Lemmas~\ref{lem:lem1}, \ref{lem:lem2}, and \ref{lem:lem3}, respectively.

\begin{figure}[t]
\centering
\includegraphics[width=7.8cm]{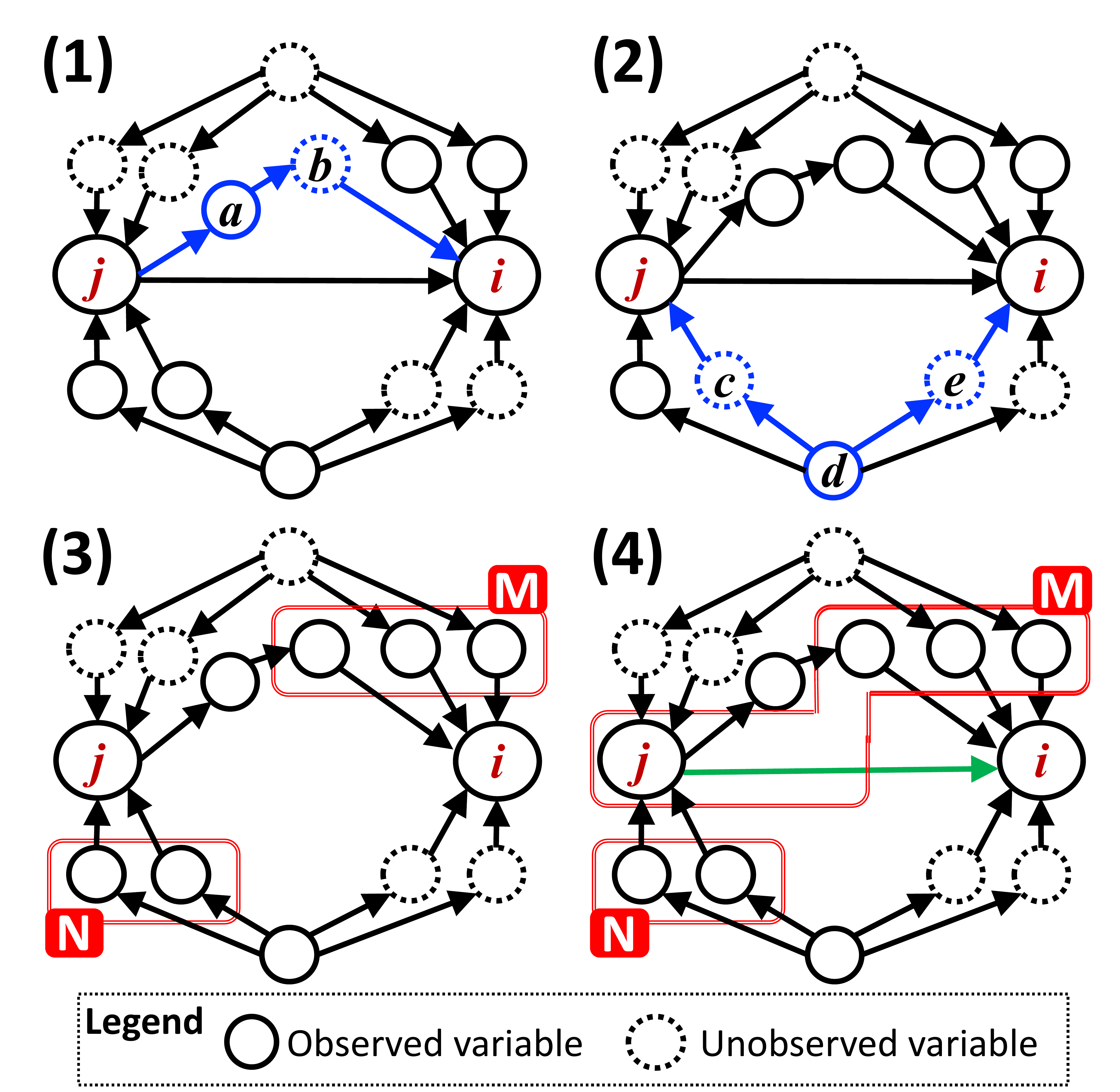}
\caption{Causal relationships in multivariate cases.}
\label{figure:lemma}
\end{figure}

\begin{lem}
\label{lem:lem1}
Assume the data generation process of the variables is CAM-UV as defined in Section~\ref{sec:problem}. If and only if Equation~\ref{eq:lem1} is satisfied, there is a UCP or UBP between $x_i$ and $x_j$ where $G_1$ and $G_2$ denote regression functions satisfying Assumption~\ref{assumption:ass2}.
\begin{align}
\begin{aligned}
\label{eq:lem1}
	&\forall G_1, G_2, M \subseteq (X \setminus \{x_i\}), N \subseteq (X \setminus \{x_j\}),\\
	& \left[\left(x_i - G_1(M)\right) \notindepe \left(x_j-G_2(N)\right) \right] 
\end{aligned}
\end{align}
Equation~\ref{eq:lem1} indicates that the residual of $x_i$ regressed on any subset of $X\setminus\{x_i\}$ and the residual of $x_j$ regressed on any subset of $X\setminus\{x_j\}$ cannot be mutually independent.
\end{lem}

\begin{exm}
	In Figure~\ref{figure:lemma}-(1), there is a UCP from $x_j$ to $x_i$ (i.e., $x_j\rightarrow x_a\rightarrow y_b\rightarrow x_i$). In Figure~\ref{figure:lemma}-(2), there is a UBP between $x_j$ and $x_i$ (i.e., $x_j\leftarrow y_c\leftarrow x_d \rightarrow y_e\rightarrow x_i$). In these cases, the common effects on $x_i$ and $x_j$ cannot be fully removed by any regression function $G_1$ or $G_2$, because the effects conveyed by the UBP or UCP cannot be removed by $G_1$ or $G_2$. 
\end{exm}

\begin{lem}
\label{lem:lem2}
Assume the data generation process of the variables is CAM-UV as defined in Section~\ref{sec:problem}. If and only if Equation~\ref{eq:lem2} is satisfied, there is no direct causal relationship between $x_i$ and $x_j$, and there is no UCP or UBP between $x_i$ and $x_j$ where $G_1$ and $G_2$ denote regression functions satisfying Assumption~\ref{assumption:ass2}.
\begin{align}
\begin{aligned}
\label{eq:lem2}
	&\exists G_1, G_2, M \subseteq (X \setminus \{x_i,x_j\}), N \subseteq (X \setminus \{x_i,x_j\}),\\
	&[(\left(x_i - G_1(M)\right) \indepe \left(x_j-G_2(N)\right))] 
\end{aligned}
\end{align}
Equation~\ref{eq:lem2} indicates that there are regression functions such that the residuals of $x_i$ and $x_j$ regressed on subsets of $X\setminus\{x_i,x_j\}$ are mutually independent.
\end{lem}

\begin{exm}
	In Figure~\ref{figure:lemma}-(3), there is no UCP or UBP between $x_i$ and $x_j$, and there is no direct causal relationship between $x_i$ and $x_j$. In Figure~\ref{figure:lemma}-(3), $M$ and $N$ are direct causes of $x_i$ and $x_j$, and they correspond to $M$ and $N$ in Equation~\ref{eq:lem2}. They block all the BPs and CPs between $x_i$ and $x_j$.
\end{exm}

\begin{lem}
\label{lem:lem3}
Assume the data generation process of the variables is CAM-UV as defined in Section~\ref{sec:problem}. If and only if Equations~\ref{eq:lem3-1} and \ref{eq:lem3-2} are satisfied, $x_j$ is a direct cause of $x_i$, and there is no UCP or UBP between $x_i$ and $x_j$ where $G_1$ and $G_2$ denote regression functions satisfying Assumption~\ref{assumption:ass2}.
\begin{align}
\begin{aligned}
\label{eq:lem3-1}
	&\forall G_1, G_2, M \subseteq (X \setminus \{x_i,x_j\}), N \subseteq (X \setminus \{x_j\}),\\
	& \left[\left(x_i - G_1(M)\right) \notindepe \left(x_j-G_2(N)\right) \right] 
\end{aligned}
\end{align}
\begin{align}
\begin{aligned}
\label{eq:lem3-2}
	&\exists G_1, G_2, M \subseteq (X \setminus \{x_i\}), N \subseteq (X \setminus \{x_i,x_j\}),\\
	& \left[\left(x_i - G_1(M)\right) \indepe \left(x_j-G_2(N)\right) \right] 
\end{aligned}
\end{align}
Equation~\ref{eq:lem3-1} indicates that the residual of $x_i$ regressed on any subset of $X\setminus\{x_i,x_j\}$ and the residual of $x_j$ regressed on any subset of $X\setminus\{x_j\}$ cannot be mutually independent. Equation~\ref{eq:lem3-2} indicates that there are regression functions such that the residual of $x_i$ regressed on a subset of $X\setminus\{x_j\}$ and the residual of $x_j$ regressed on a subset of $X\setminus\{x_i,x_j\}$ are mutually independent.

\end{lem}

\begin{exm}
\label{exm:exm3}
	In Figure~\ref{figure:lemma}-(4), no UCP or UBP exists between $x_j$ and $x_i$. There is a direct causal relationship between $x_j$ and $x_i$. In Figure~\ref{figure:lemma}-(4), $M$ and $N$ are direct causes of $x_i$ and $x_j$, and they correspond to $M$ and $N$ in Equation~\ref{eq:lem3-2}. They block all the BPs and CPs between $x_i$ and $x_j$ including the direct causal effect of $x_j$ on $x_i$ (i.e., $x_j\rightarrow x_i$).
\end{exm}

Although it is impossible to identify the causal relationship between $x_i$ and $x_j$ when there is a UCP or UBP, it is possible to avoid the incorrect determination of the causal relationship if we use Lemma~\ref{lem:lem1}. If there is no UCP or UBP, it is possible to identify the direct causal relationship between $x_i$ and $x_j$ using Lemmas~\ref{lem:lem2} and~\ref{lem:lem3}.

Next, we provide Lemma~\ref{lem:lem4}, which can be used for identifying a $sink$ of a set of observed variables. Let $K$ denote a set of observed variables. Observed variable $x_i$ is called a {\it sink} of $K$ when $x_i\in K$ holds, and each $x_j\in K\setminus\{x_i\}$ is not a descendant of $x_i$. Example~\ref{exm:4} is provided after Lemma~\ref{lem:lem4}.

\begin{lem}\label{lem:lem4}
	Assume the data generation process of the variables is CAM-UV as defined in Section~\ref{sec:problem}. Let $K$ denote a set satisfying $K\subseteq X$ and assume $x_i\in K$. If Equation~\ref{eq:lem4} holds, each $x_j\in K\setminus\{x_i\}$ is not a descendant of $x_i$ where $G_i^1$, $G_j^1$, $G_i^2$, and $G_j^2$ denote regression functions satisfying Assumption~\ref{assumption:ass2}.
\begin{align}
\begin{aligned}\label{eq:lem4}
	&\exists G_i^1, M_i\subseteq (X\setminus K),\\
	&\forall x_j \in (K\setminus\{x_i\}),\exists G_j^1,M_j\subseteq (X\setminus K),\forall G_i^2,G_j^2\\
	&[((x_i-G_i^1(M_i\cup K\setminus \{x_i\}))\indepe (x_j-G_j^1(M_j)))\\
	&\land((x_i-G_i^2(M_i))\notindepe (x_j-G_j^2(M_j)))]
\end{aligned}
\end{align}
Equation~\ref{eq:lem4} indicates that there exists set $M_j\subseteq (X\setminus K)$ for each $x_j\in K$ satisfying the condition that the residual of $x_i$ regressed on $(M_i\cup K\setminus \{x_i\})$ is independent of the residual of $x_j$ regressed on $M_j$ for each $x_j\in K\setminus\{x_i\}$. In addition, the residual of $x_i$ regressed on $M_i$ cannot be independent of the residual of $x_j$ regressed on $M_j$ for each $x_j\in K\setminus\{x_i\}$.
\end{lem}

\begin{exm}\label{exm:4}
In Figure~\ref{figure:methodgraph},	$K$ consists of three observed variables (i.e., $K=\{x_a,x_b,x_i\}$). The BP between $x_i$ and $x_a$ and the BP between $x_i$ and $x_b$ are blocked by $M_a$ and $M_b$ respectively. In addition, all the effects of $x_a$ and $x_b$ on $x_i$ are mediated by the direct causes of $x_i$ which are included in $K\setminus\{x_i\}$. Then, the residual of $x_i$ regressed on $M_i\cup K\setminus\{x_i\}$ can be independent of the residuals of $x_a$ and $x_b$ regressed on $M_a$ and $M_b$ respectively. In addition, the residual of $x_i$ regressed on $M_i$ cannot be independent of the residuals of $x_a$ and $x_b$ regressed on $M_a$ and $M_b$ respectively. These statements are formulated as Equation~\ref{eq:model1}, which can be generalized to Equation~\ref{eq:lem4}.
\begin{align}
\begin{aligned}\label{eq:model1}
	&((x_i-G_i^1(M_i\cup K\setminus\{x_i\}))\indepe (x_a-G_a^1(M_a)))\\
	&\land((x_i-G_i^1(M_i\cup K\setminus\{x_i\}))\indepe (x_b-G_b^1(M_b)))\\
	&\land((x_i-G_i^2(M_i))\notindepe (x_a-G_a^2(M_a)))\\
	&\land((x_i-G_i^2(M_i))\notindepe (x_b-G_b^2(M_b)))
\end{aligned}
\end{align}

\end{exm}

\begin{figure}[t]
\centering
\includegraphics[width=7.7cm]{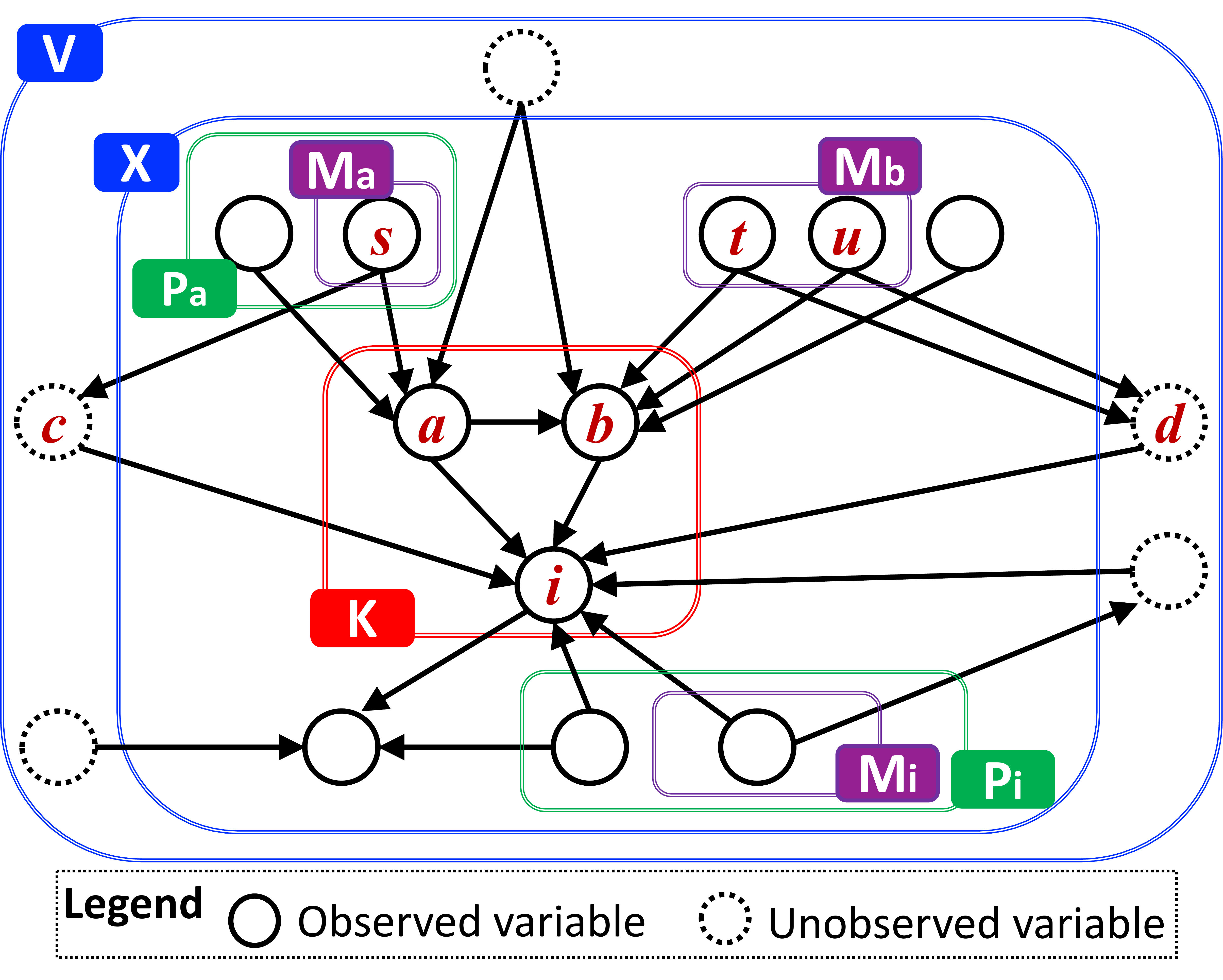}
\caption{Identification of $x_i$ as a sink of $K$.}
\label{figure:methodgraph}
\end{figure}

\section{Model estimation}
\label{section:proposed}

We propose a method to infer causal relationships between observed variables. The causal graphs inferred by our proposed method include directed edges and undirected dashed edges (see Figure~\ref{figure:intro}-(c)). A directed edge indicates a direct causal relationship, and an undirected dashed edge indicates that there is a UCP or UBP between the variables.

\begin{algorithm}[]
\SetKwProg{init}{initialization}{}{}
\DontPrintSemicolon
\KwIn{i.i.d samples of a $p$-dimensional distribution on $\{x_1,\cdots,x_p\}$ $X$, maximum number of variables to examine causality for each step $d$, significance level for independence test $\alpha$.}
\KwOut{the sets of the parents $\{M_1, \cdots, M_p \}$.}
\SetKwBlock{Begin}{function}{end function}
\Begin($\text{getDirectedEdges} {(} X, d, \alpha {)}$)
{
	{\sf\# PHASE 1: Extracting the candidates of the parents of each variable.}\;
	    \For{$i=1$ {\bf to} $p$}{
    	{\bf Initialize} $M_i \leftarrow \emptyset$.\;
    }
    {\bf Initialize} $t \leftarrow 2$.\;
    \While{$t\leq d$}{
    	{\bf Initialize} $noChange \leftarrow {\rm True}$.\;
    	\ForEach{$K\in\{K|K\subseteq X, |K|=t\}$}{
			{\sf\# Finding the most endogenous variable $x_b$ in {\it K}}\;
			$x_b\leftarrow \argmax_{x_i\in K}\reallywidehat{\text{p-HSIC}}(x_i-G_1(M_i\cup K\setminus\{x_i\}),\{x_j-G_2(M_j)| x_j\in K\setminus\{x_i\}\})$\;
			{\sf\# Computing the independence between the residuals in regard to Lemma~\ref{lem:lem4}}\;
			$e\leftarrow\reallywidehat{\text{p-HSIC}}(x_b-G_1(M_b\cup K\setminus\{x_b\}),\{x_j-G_2(M_j)| x_j\in K\setminus\{x_b\}\})$\;
			$h\leftarrow\displaystyle\max_{x_j\in K\setminus\{x_b\}}\reallywidehat{\text{p-HSIC}}(x_b-G_1(M_b), x_j-G_2(M_j))$\;
			{\sf\# Checking whether $x_b$ is really a sink of {\it $K$}}\;
			\If{$(\alpha < e)\land(\alpha > h)$}{
				{\sf\# When $x_b$ is a sink of $K$, add each variable in $K\setminus\{x_b\}$ to $M_b$}\;
				$M_b \leftarrow M_b\cup (K\setminus\{x_b\})$\;
				$noChange \leftarrow {\rm False}$\;
			}
    	}
    	{\sf\# If each $M_i$ remains unchanged, increment $t$ by one. If not, substitute $2$ for $t$.}\;
    	\uIf{$noChange={\rm True}$}{
    		$t\leftarrow t+1$\;
    	}\Else{
    		$t\leftarrow 2$\; 
    	}
    }
    {\sf\# PHASE 2: Determining the parents of each variable.}\;
    \For{$i=1$ {\bfseries to} $p$}{
    	\ForEach{$x_j\in P_i$}{
    		{\sf\# Checking whether $x_j$ is parent of $x_i$}\;
    		\If{$\alpha < \reallywidehat{\text{\rm p-HSIC}}(x_i-G_1(M_i\setminus \{x_j\}),x_j-G_2(M_j))$}{
	    		{\sf \# When $x_j$ is not a parent, remove it from $M_i$}\;
    			$M_i\leftarrow M_i\setminus \{x_j\}$\;
    		}
    	}
    }
    \Return{$\{M_1, \cdots, M_p \}$}
}
\caption{Determine the directed edges}
   \label{algo:algorithm}
\end{algorithm}

First, we propose a method to determine the directed edges. The detailed procedure is listed in Algorithm~\ref{algo:algorithm}, which consists of two steps. Our method first extracts the candidates of the parents of each observed variable (lines 2--23 in Algorithm~\ref{algo:algorithm}), then it determines the parents of each observed variable (lines 24--30). The notations $G_1$ and $G_2$ in Algorithms~\ref{algo:algorithm} and \ref{algo:algorithm2} indicate GAM regression functions. Those functions perform differently in different lines and different iterations.

The first step of Algorithm~\ref{algo:algorithm} involves each $M_i$ collecting observed variables that are not descendants of $x_i$, which we call the candidates of the parents of $x_i$. The method first initializes each $M_i$ to an empty set (lines 3--4 in Algorithm~\ref{algo:algorithm}). Then it repeats finding a sink for each $K$ that satisfies $K\subseteq X$ and $|K|=t$ (lines 8--18). That is, each $K$ is a set consisting of $t$ observed variables. The value of $t$ starts at $2$ (line 5). It is incremented by 1 when each $M_i$ remains unchanged through an iteration, and it is updated to 2 when at least one $M_i$ changes during the iteration (lines 20--23). When our method determines that $x_b$ is a sink of $K$, it updates $M_b$ by adding each variable in $K\setminus\{x_b\}$ to $M_b$ (lines 16--17). The iteration ends when $t$ exceeds $d$ (line~6), which is a hyperparameter that is set as the maximum number of $|K|$. The purpose of $d$ is to reduce the computation time, and it should be set according to the sparsity of the causal relationships.\\
To find a sink for each $K$, our proposed method first finds the most endogenous variable $x_b$ in $K$ (lines 9--10). Such a $x_b$ maximizes the independence between $(x_b-G_1(M_b\cup K\setminus\{x_b\}))$ and $(\{x_j-G_2(M_j)| x_j\in K\setminus\{x_b\}\})$. We use the {\it p}-value of the Hilbert--Schmidt Independence Criteria (HSIC)~\citep{NIPS2007_3201} for measuring independence, and we also use the GAM regression method proposed by~\citet{wood2004}. Our method examines whether $x_b$ and the other variables in $K\setminus\{x_b\}$ satisfy the condition defined in Lemma~\ref{lem:lem4} using the significance level for independence test, given as hyperparameter $\alpha$ (lines 11--15). If $x_b$ and $K\setminus\{x_b\}$ satisfy the condition, then each variable in $K\setminus\{x_b\}$ is added to $M_b$ (lines 16--17).

In the second step, our proposed method determines the parents of each observed variable. If $x_j\in M_i$ satisfies $x_i-G_1(M_i\setminus \{x_j\})\indepe x_j-G_2(M_j)$, it is not a parent of $x_i$ because of Lemma~\ref{lem:lem2}. Therefore, our method removes each $x_j$ satisfying the above equation from $M_i$ and defines the variables remaining in $M_i$ as the parents of $x_i$ (lines 27--30). The reason why the variables remaining in $M_i$ are the parents of $x_i$ is as follows. Each directed path from each $x_j$ in $M_i$ to $x_i$ is blocked by the parents of $x_i$ that are included in $M_i$ (i.e. $M_i \cap P_i$). If $x_j$ is not a parent of $x_i$, all the directed paths from $x_j$ to $x_i$ is blocked by $(M_i \cap P_i\setminus \{x_j\})$. Then, $(x_i-G_1(M_i\setminus\{x_j\}))$ and $(x_j-G_2(M_j))$ are mutually independent. If $x_j$ is a parent of $x_i$, there is a direct causal effect $x_j \rightarrow x_i$, and it is not blocked by $(M_i\cap P_i\setminus \{x_j\})$. Then, $(x_i-G_1(M_i\setminus \{x_j\}))$ and $(x_j-G_2(M_j))$ cannot be mutually independent. Therefore, the variables remaining in $M_i$ are parents of $x_i$.

After determining the direct causal relationships, the proposed method determines variable pairs having UBPs or UCPs (i.e., variable pairs connected with dashed undirected edges). The detailed procedure is listed in Algorithm~\ref{algo:algorithm2}. If the residual of $x_i$ regressed on $M_i$ and that of $x_j$ regressed on $M_j$ are mutually dependent, there is a UCP or UBP between them (lines 5--6 in Algorithm~\ref{algo:algorithm2}). Therefore, our proposed method connects $x_i$ and $x_j$ with a dashed undirected edge.

\begin{algorithm}[t]
\SetKwProg{init}{initialization}{}{}
\DontPrintSemicolon
\KwIn{i.i.d samples of a $p$-dimensional distribution on $\{x_1,\cdots,x_p\}$ $X$, series of the sets of the parents $\{M_1,\cdots,M_p\}$, significance level for independence test $\alpha$.}
\KwOut{set of variable pairs having a UCP or UBP $C$.}
\SetKwBlock{Begin}{function}{end function}
\Begin($\text{getUndirectedEdges} {(} X, \{M_1,\cdots,M_p\}, \alpha {)}$)
{
	{\bf Initialize} $C \leftarrow \emptyset$.\;
	\ForEach {$i,j\ ((1\leq i,j\leq p)\land (i\neq j))$}{
		\If{$(x_i\notin M_j)\land (x_j \notin M_i)$}{
			
			\If{$\alpha > \reallywidehat{\text{\rm p-HSIC}}(x_i-G_1(M_i),x_j-G_2(M_j))$}{
				$C\leftarrow C\cup\{\{x_i,x_j\}\}$\;
			}
		}
	}
	\Return{$C$}
}
\caption{Determine the undirected dashed edges}
   \label{algo:algorithm2}
\end{algorithm}

The time complexity of the method is $\displaystyle\mathcal{O}\left( p2^pn^2 \right)$ when $d$ (the maximum number of $|K|$) equals the number of all the observed variables $p$ (i.e., $d=p$). Please refer to Supplementary materials for the details.

\section{Experiments}

\begin{figure*}[t]
\begin{center}
  \includegraphics[width=14.9cm]{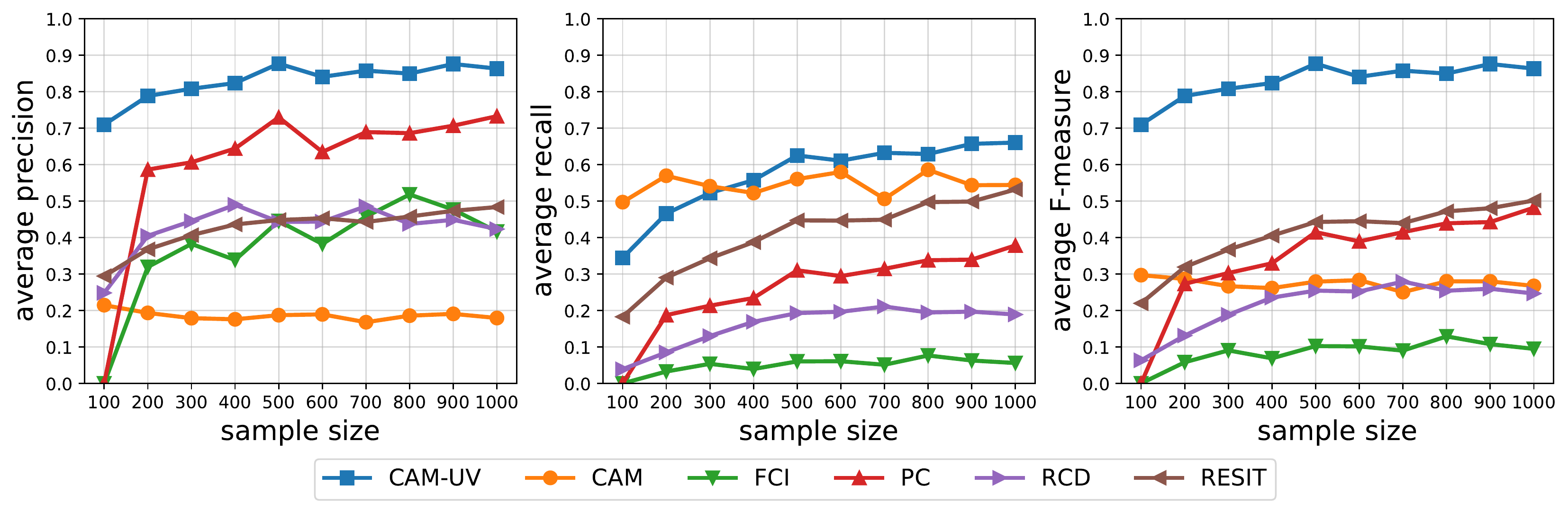}
\end{center}
\caption{Performance on artificial data generated from causal additive models with unobserved variables.}
\label{figure:simresult}
\end{figure*}

\begin{figure}[t]
\centering
\includegraphics[width=6.2cm]{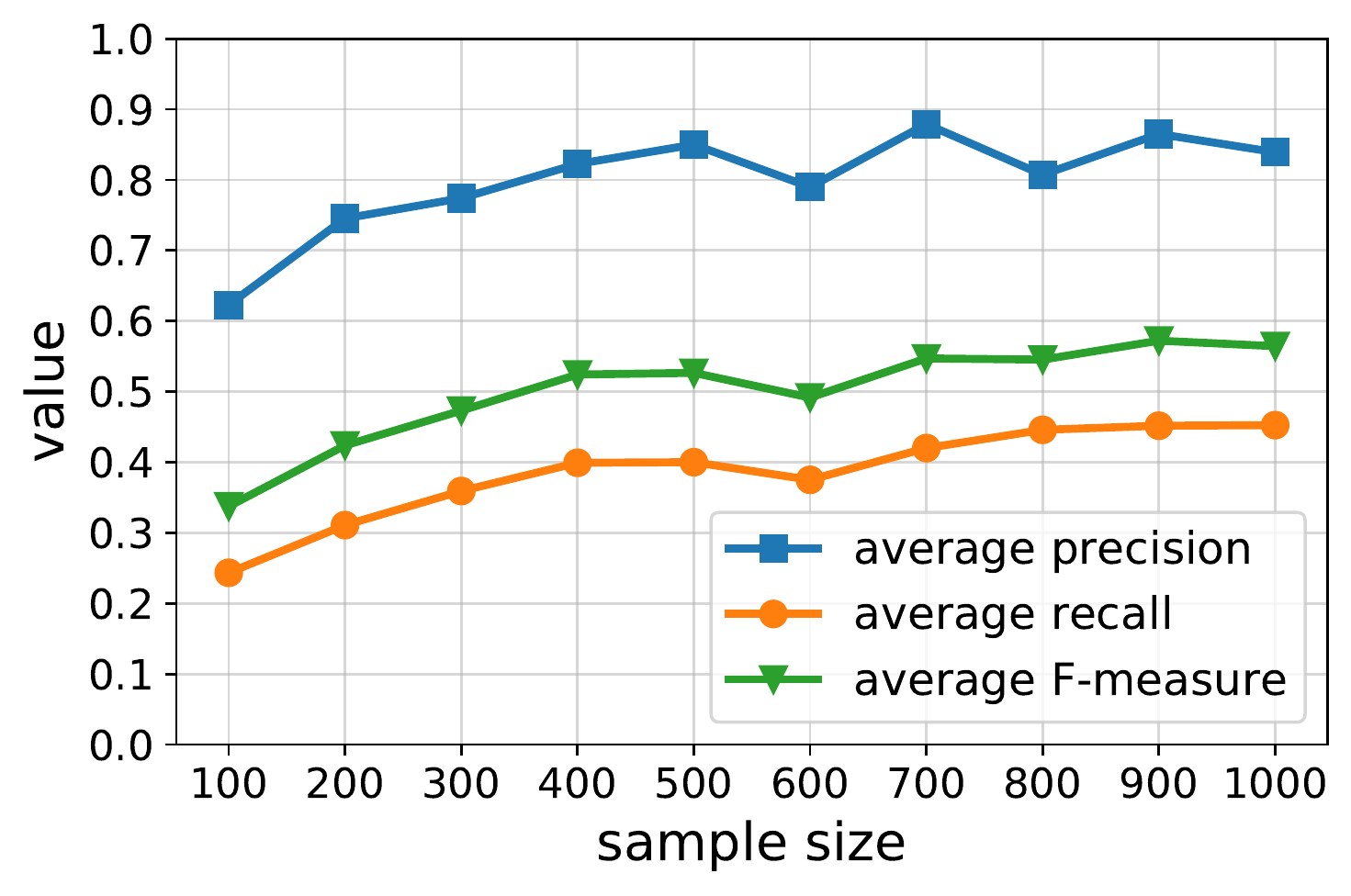}
\caption{Performance of our method CAM-UV in identifying variable pairs having a UCP or a UBP.}
\label{figure:UCPUBP}
\end{figure}

\begin{figure}[t]
\centering
\includegraphics[width=6.2cm]{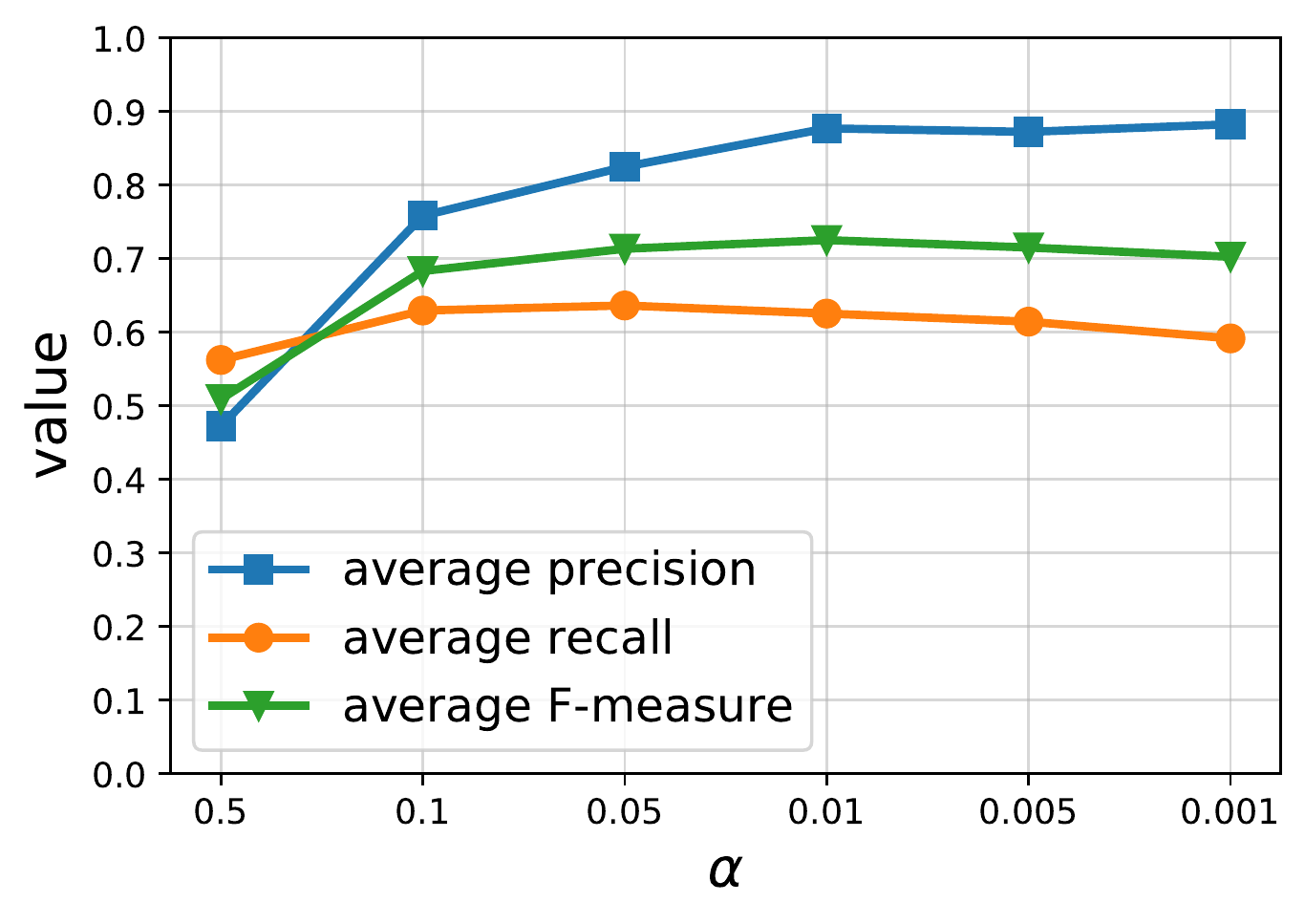}
\caption{Sensitivity of our method CAM-UV to the setting of $\alpha$ in identifying direct causal relationships.}
\label{figure:sensitivity}
\end{figure}

We compared the performance of our method to the following methods: PC~\citep{Spirtes91}, FCI~\citep{fci}, CAM~\citep{buhlmann2014}, RESIT~\citep{peters2014}, and RCD~\citep{maeda20a}. PC is a constraint-based method that assumes the absence of unobserved variables. FCI is also a constraint-based method, but it assumes the presence of unobserved variables. CAM and RESIT are causal functional model-based methods that assume that causal functions are nonlinear and unobserved variables are absent. In contrast, RCD is a causal functional model-based method that assumes that causal functions are linear and unobserved variables are present.

The true causal graphs used for the evaluation are defined such that a directed edge is drawn from $x_j$ to $x_i$ when there is a directed path from $x_j$ to $x_i$ on which no other observed variable exists (see Figures~\ref{figure:intro}-(a,b)). There are types of edges other than directed edges (i.e., $\leftarrow$) in the graphs produced by the above methods and our proposed method, but we only used directed edges for the comparative evaluation.

We used precision, recall, and the F-measure as the evaluation measures. Avoiding false inferences is very important in causal discovery. By evaluating the results in terms of precision, recall, and F-measure, it is possible understand how well each method avoids false inferences. The true positive (TP) is the number of true directed edges that a method correctly infers in terms of their positions and directions. Precision represents the TP divided by the number of estimations, and recall represents the TP divided by the number of all true directed edges. Furthermore, the F-measure is defined as $\text{F-measure} = 2 \cdot \text{precision} \cdot \text{recall} / (\text{precision} + \text{recall})$.

We set the significance levels required for the baseline methods and our proposed method as $0.01$. In addition, we set the maximum number of $|K|$ to 3 (i.e., $d=3$) for our proposed method (see Algorithm~\ref{algo:algorithm}).

We conducted experiments on the artificial data generated from CAM-UV and the simulated fMRI data created in \citet{SMITH2011875}.

\subsection{Performance on artificial data generated from CAM-UV}
\label{sec:artifitial}

\begin{figure*}[t]
\centering
\includegraphics[width=14.4cm]{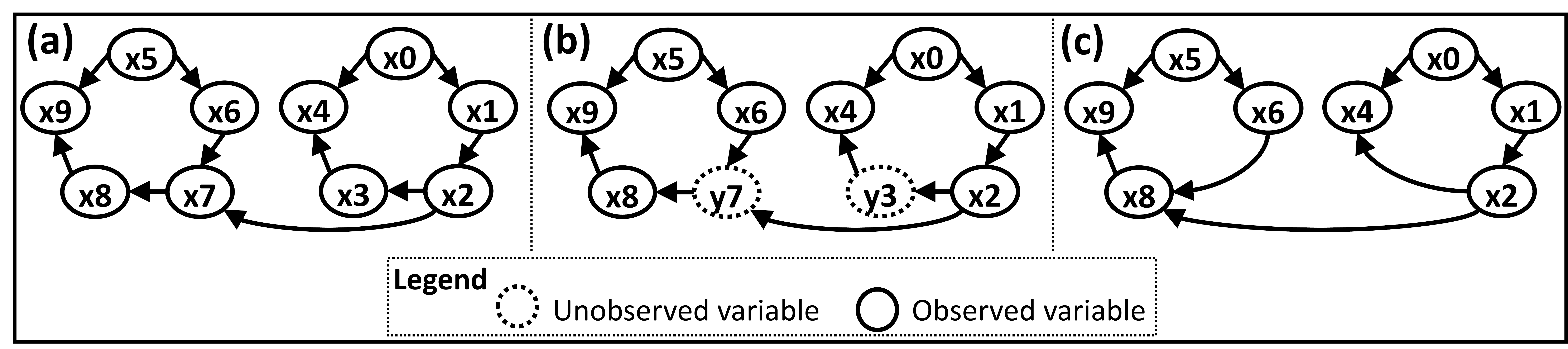}
\caption{(a) Causal structure in fMRI data (b) Omitted variables. (c) True causal graph after omitting variables.}
\label{figure:fmridata}
\end{figure*}

\begin{figure*}[t]
\begin{center}
  \includegraphics[width=15.6cm]{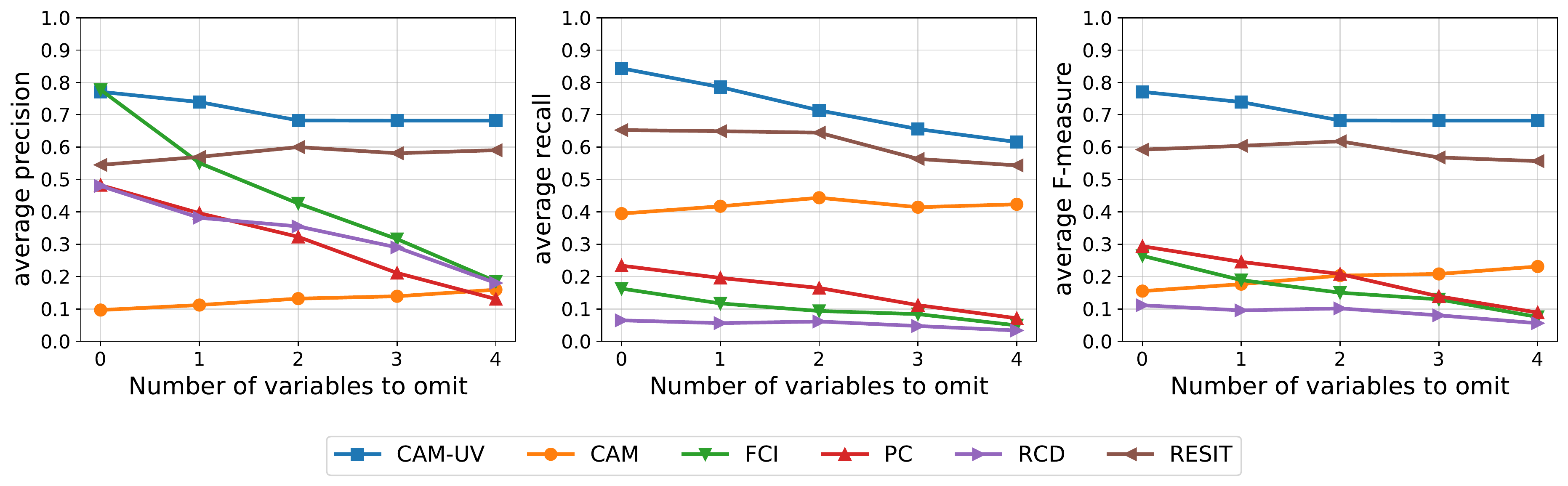}
\end{center}
\caption{Performance on simulated fMRI data.}
\label{figure:fmriresult}
\end{figure*}

{\bf Comparison with baseline methods:} We performed 100 experiments using artificial data with each sample size $n\in\{100, 200, \cdots, 900, 1000\}$ to compare our method to existing methods. The data for each experiment were generated as follows. The data generation process was accomplished using Equation~\ref{eq:obdef2}. We prepared ten observed variables, two unobserved common causes, and two unobserved intermediate causal variables. The causal order of the observed variables was determined the same as the order of the indices of the observed variables. The direct causal relationships between the observed variables were determined based on the Erd\H{o}s--R\'{e}nyi model~\citep{erdHos1960evolution} with parameter 0.3. That is, each variable pair was connected by an edge with a probability of 0.3. The directions of the edges were determined according to the causal order. We drew two directed edges from each unobserved common cause to two randomly selected observed variables. Finally, two variable pairs were randomly chosen, and an unobserved intermediate causal variable was inserted between each variable pair. The indices of the observed variables were randomly permuted after the data were created. The value of each $v_i$ defined in Equation~\ref{eq:obdef2} was determined by 
\begin{align}
	\begin{aligned}
	\label{eq:exp01}
		v_i = \frac{h_i}{{\rm sd}(h_i)} ,\ \ h_i = \sum_{v_j\in (P_i\cup Q_i)}\left((v_j+a_{i,j})^{c_{i,j}}+b_{i,j}\right)+ n_i,
	\end{aligned}
\end{align}

where $a_{i,j}$, $b_{i,j}$, and $c_{i,j}$ denote constants, $n_i$ denotes a random variable, and ${\rm sd}(h_i)$ denotes the standard deviation of $h_i$. The values of $a_{i,j}$ and $b_{i,j}$ were randomly chosen from $U(-5,5)$ and $U(-1,1)$, respectively. The value of $c_{i,j}$ was randomly selected from $\{2,3\}$, where the probability of selecting either value is $0.5$. The samples of $n_i$ were taken from $U(-10+d_i,10+d_i)$ where $d_i$ is a constant randomly chosen from $U(-2,2)$. The causal effect of $v_j$ on $v_i$ (i.e., $f_j^{(i)}(v_j)$ in Equation~\ref{eq:obdef2}) corresponds to $\left((v_j+a_{i,j})^{c_{i,j}}+b_{i,j}\right)/{\rm sd}(h_i)$.

Figure~\ref{figure:simresult} shows the results. The graphs plot the mean values of the evaluation measures. CAM-UV scores the best in terms of precision and the F-measure for each sample size. The recall value of CAM-UV increases as the sample size increases. When the sample size is 300 or less, the scores of our proposed method are the second best next to CAM, but it scores the best when the sample size is more than 300.

{\bf Performance of identifying UCPs and UBPs:} Figure~\ref{figure:UCPUBP} shows how well our proposed method identified UCPs and UBPs. The true positive (TP) is the number of variable pairs having a UCP or UBP and those that are connected by dashed undirected edges in the causal graph inferred by our proposed method. The graphs in Figure~\ref{figure:UCPUBP} show that the precision, recall, and F-measure values increase as the sample size increases.

{\bf Sensitivity to the hyperparameter:} We conducted experiments to investigate the sensitivity of the proposed method to the settings of the hyperparameter $\alpha$. We used 500 samples for each experiment with $\alpha\in\{0.5,0.1,0.05,0.01,0.005,0.001\}$. Figure~\ref{figure:sensitivity} shows the results. The precision and F-measure values gradually increase as $\alpha$ decreases, but they remain flat for $\alpha\leq 0.01$.

{\bf Average runtime:} The average runtime of the proposed method was 8.3 seconds when the sample size was 500 and 24.9 seconds when the sample size was 1000. The details of the average runtimes and the machine used for computing are available in the Supplementary materials.

\subsection{Performance on simulated fMRI data}

We conducted experiments on simulated fMRI data generated by \citet{SMITH2011875} based on a well-known mathematical model of interactions among brain regions, the dynamic causal model~\citep{FRISTON20031273}. We used one of their datasets (``sim2'') with ten variables, the causal relationships of which are shown in Figure~\ref{figure:fmridata}-(a). We randomly omitted $m$ variables for each experiment to create a dataset with unobserved variables. For example, when $m=2$ and $x_3$ and $x_7$ are omitted to make unobserved variables $y_3$ and $y_7$, as shown in Figure~\ref{figure:fmridata}-(b), the causal graph for evaluation includes directed edges $x_2\rightarrow x_4$, $x_2\rightarrow x_8$, and $x_6\rightarrow x_8$, as shown in Figure~\ref{figure:fmridata}-(c). We conducted 100 experiments with 1000 samples randomly extracted from the data for each $m\in\{0,1,2,3,4\}$. Figure~\ref{figure:fmriresult} shows the results. Though the precision score for CAM-UV is slightly lower than FCI when $m=0$, our method scores the best for the other cases.

\section{Conclusions}

In this study, we extended causal additive models to incorporate unobserved variables, the model for which we called {\it causal additive models with unobserved variables} (CAM-UV). Our theoretical analysis showed that the direct causal relationships between observed variables cannot be determined when there is an {\it unobserved causal path} (UCP) or {\it unobserved backdoor path} (UBP) between the variables. However, the theoretical results also show that it is possible to identify such variable pairs and to avoid incorrect inferences. Based on these theoretical results, we proposed a method to infer causal graphs for CAM-UV and verified the method through experiments. As demonstrated by our theoretical and experimental results, our proposed method is effective in inferring causal relationships in the presence of unobserved variables. Our future research will focus on the application of our method for the efficient intervention design using the results of UBPs and UCPs.

%

\bibliography{ref.bib}


\newpage
\onecolumn
\setcounter{section}{0}
\renewcommand{\thesection}{\Alph{section}}
\setcounter{lem}{0}
\renewcommand{\thelem}{\Alph{lem}}
\setcounter{figure}{0}
\renewcommand{\thefigure}{\Alph{figure}}

\begin{center}
{\LARGE \bf Supplementary material for the manuscript: \\``Causal Additive Models with Unobserved Variables''}
\end{center}

\section{Proofs}

This appendix provides the proofs of Lemmas~1--4. First, we provide Lemma~\ref{lem:hosoku}, which is used in the proofs of Lemmas~1--3.

\begin{lem}\label{lem:hosoku}
Assume the data generation process of the variables is CAM-UV as defined in Section~\ref{sec:problem}. Let $M$ denote a set of variables not containing $v_i$ (i.e., $M\subseteq V\setminus\{v_i\}$), and let $s(v_i)$ denote a linear or nonlinear function of $v_i$. Then, the residual of $s(v_i)$ regressed on $M$ cannot be independent of $n_i$ as formulated in 
\begin{equation}\label{eq:hosoku}
	\forall M\subseteq V\setminus\{v_i\}, G, [s(v_i)-G(M)\notindepe n_i]
\end{equation}
where $G$ denotes a function satisfying Assumption~\ref{assumption:ass2}.

\end{lem}
\begin{proof}
We prove the lemma by contradiction. Assume that $s(v_i)-G(M)\indepe n_i$ holds. Function $G(M)$ is formulated as $G(M)=\sum_{x_m\in M}g_{m}(x_m)$ where each $g_{m}(x_m)$ is a nonlinear function of $x_m$, according to Assumption~\ref{assumption:ass2}. Then there is a variable $v_j\in M$ such that $n_i$ in $s(v_i)$ is canceled out by $g_j(v_j)$, so $v_j$ satisfies $v_j\notindepe n_i$. The variable $v_j$ is a descendant of $v_i$ because $v_j\indepe n_i$ holds when $v_j$ is not a descendant of $v_i$. Furthermore, $v_j$ can be formulated as $v_j=u(n_i, T)+n_j$, where $u$ is a function, and $T$ is a set of all the external effects influencing $v_j$, except for $n_i$ and $n_j$. Because $g_j$ is a nonlinear function, $g_j(v_j)=g_j(u(n_i, T)+n_j)$ has a term involving a function of $n_i$ and $n_j$, which can be formulated as $a(n_i,n_j)$, but it cannot be represented as the linear sum of functions of $n_i$ and $n_j$, such as $a(n_i,n_j)=b(n_i)+c(n_j)$. Because $v_i$ is an ancestor of $v_j$, $v_i\indepe n_j$ holds. Then, because $v_i$ does not have a term involving a function of $n_i$ and $n_j$, terms containing $n_i$ cannot fully be removed from $s(v_i)-g_j(v_j)$. Thus, Equation~\ref{eq:hosoku} holds.

\end{proof}

\subsection{Proof of Lemma~1}
\begin{proof}

As defined in Section~\ref{sec:problem}, $x_i$ is formulated as $x_i=z_i+w_i+n_i$, where $z_i$ is the sum of the direct effects of observed variables on $x_i$, $w_i$ is the sum of the direct effects of unobserved variables on $x_i$, and $n_i$ is the external effect on $x_i$. In addition, $z_i$ and $z_j$ are the mixtures of the observed direct causes of $x_i$ and $x_j$, with the causal functions that take the form of generalized additive models. We represent $G_1(M)$ and $G_2(N)$ as $G_1(M)=z_i+G_1^{*}(M)$ and $G_2(N)=G_2^{*}(N)+z_j$. Then Equation~\ref{eq:lem1} is equivalent to 
\begin{equation}\label{eq:pr1-1}
	\forall G^{*}_1, G^{*}_2, M \subseteq (X \setminus \{x_j\}), N \subseteq (X \setminus \{x_i\}), \left[\left(w_i + n_i - G^{*}_1(M)\right) \notindepe \left(w_j + n_j-G^{*}_2(N)\right) \right]. 
\end{equation}

We define $A_i$ and $A_j$ as the sets of all the external effects depending on $(w_i+n_i)$ and $(w_j+n_j)$, respectively (i.e., $A_i=\{n_k|n_k\notindepe (w_i+n_i)\}$ and $A_j=\{n_k|n_k\notindepe (w_j+n_j)\}$). Then, $A_i$ is the set of the external effects of all the unobserved direct causes of $x_i$ and the external effect of $x_i$. Because $M$ is a set of observed variables excluding $x_i$, $M$ does not contain $x_i$ or each unobserved direct cause of $x_i$. Then, because of Lemma~\ref{lem:hosoku}, $G_1(M)$ cannot cancel out each $n_k\in A_i$ from $(w_i+n_i)$. In the same way, $G_2(N)$ cannot cancel out each $n_k\in A_j$ from $(w_j+n_j)$. Therefore, Equation~\ref{eq:pr1-1} is equivalent to 

\begin{equation}\label{eq:pr1-2}
	(w_i+n_i)\notindepe(w_j+n_j).
\end{equation}

Because CAM-UV assumes that all the external effects are mutually independent, $n_i\indepe n_j$ holds. Then we obtain
	\begin{align}
		\begin{aligned}
		\label{eq:proof1-3}
		(w_i \notindepe n_j) \lor (n_i \notindepe w_j)\lor (w_i \notindepe w_j).
		\end{aligned}
	\end{align}

If $(w_i \notindepe n_j)$ holds, there exists unobserved variable $y_k\in Q_i$ such that $y_k$ has a term involving a function of $n_j$. Then, there is a directed path from $x_j$ to $x_i$ that ends with the directed edge connecting $y_k$ and $x_i$ (i.e., $x_j\rightarrow \cdots \rightarrow y_k \rightarrow x_i$). Therefore, there exists a UCP from $x_j$ to $x_i$. In addition, the existence of a UCP from $x_j$ to $x_i$ also implies that $(w_i \notindepe n_j)$ because of Assumptions~\ref{assumption:as1} and \ref{assumption:ass2}. Therefore, if and only if $(w_i \notindepe n_j)$ holds, there is a UCP from $x_j$ to $x_i$. Similarly, if and only if $(w_j \notindepe n_i)$ holds, there is a UCP from $x_i$ to $x_j$. When $(w_i \notindepe w_j)$ holds, there exists $n_k$ satisfying $n_k\notindepe w_i$ and $n_k\notindepe w_j$. When $n_k\notindepe w_i$ holds, $v_k$ is an ancestor of an unobserved direct cause of $x_i$ or $v_k$ is an unobserved direct cause of $x_i$. Similarly, when $n_k\notindepe w_j$ holds, $v_k$ is an ancestor of an unobserved direct cause of $x_j$ or $v_k$ is an unobserved direct cause of $x_j$. Therefore, there is a UBP between $x_i$ and $x_j$ when $(w_i \notindepe w_j)$ holds. In addition, the existence of a UBP between $x_i$ and $x_j$ also implies that $(w_i \notindepe w_j)$ because of Assumptions~\ref{assumption:as1} and \ref{assumption:ass2}. Therefore, if and only if Equation~\ref{eq:lem1} is satisfied, there is a UCP or UBP between $x_i$ and $x_j$.
	\end{proof}

\subsection{Proof of Lemma~2}
\begin{proof}
When Equation~\ref{eq:lem2} holds, Equation~\ref{eq:lem1} does not hold. Therefore, there is no UCP or UBP between $x_i$ and $x_j$ according to Lemma~\ref{lem:lem1}. We prove that there is no direct causal relationship between $x_i$ and $x_j$ by contradiction. Assume that $x_j$ is a direct cause of $x_i$ and that $G_1$, $G_2$, $M \subseteq (X \setminus \{x_i,x_j\})$, and $N \subseteq (X \setminus \{x_i,x_j\})$ satisfy 
\begin{equation}\label{eq:proof2-1}
(x_i - G_1(M)) \indepe (x_j-G_2(N)).
\end{equation}
Then $x_i$ has terms involving functions of $x_j$ when $x_j$ is a direct cause of $x_i$:

\begin{align}
	\begin{aligned}\label{eq:proof2-2}
		x_i=f_j^{(i)}(x_j)+ \sum_{x_m \in (P_i\setminus\{x_j\})} f_m^{(i)}(x_m)+w_i+n_i.
	\end{aligned}
\end{align}

Both $M$ and $N$ do not contain $x_i$ or $x_j$. Because of Lemma~\ref{lem:hosoku}, $(x_i - G_1(M)) \notindepe n_j$ and $(x_j - G_2(N)) \notindepe n_j$ hold. Then, $(x_i - G_1(M)) \notindepe (x_j-G_2(N))$ holds because of Assumption~\ref{assumption:ass2}. However, this contradicts Equation~\ref{eq:proof2-1}. Therefore, there is no direct causal relationship between $x_i$ and $x_j$.
\end{proof}

\subsection{Proof of Lemma~3}
\begin{proof}
When Equation~\ref{eq:lem3-2} holds, Equation~\ref{eq:lem1} does not hold. Therefore, there is no UCP or UBP between $x_i$ and $x_j$ according to Lemma~\ref{lem:lem1}. When Equation~\ref{eq:lem3-1} holds, Equation~\ref{eq:lem2} does not hold. Therefore, there is a direct causal relationship between $x_i$ and $x_j$ according to Lemma~\ref{lem:lem2}. In the following, we consider whether the direction of the causal effect between $x_i$ and $x_j$ is identifiable.

Assume that $x_j$ is a direct cause of $x_i$. As defined in Section~\ref{sec:problem}, $x_i$ is formulated as $x_i=z_i+w_i+n_i$, where $z_i$ is the sum of the direct effects of observed variables on $x_i$, $w_i$ is the sum of the direct effects of unobserved variables on $x_i$, and $n_i$ is the external effect on $x_i$. Then, $z_i$ includes the effect of $x_j$, and $z_j$ does not include the effect of $x_i$. Because $M$ can contain $x_j$, we can set $G_1(M)$ as $G_1(M)=z_i$ and set $G_2(N)$ as $G_2(N)=z_j$. Then we obtain
\begin{equation}\label{eq:proof3-1}
x_i-G_1(M)=w_i+n_i,	
\end{equation}
\begin{equation}\label{eq:proof3-2}
x_j-G_2(N)=w_j+n_j.
\end{equation}
When there is no UCP from $x_j$ to $x_i$, $w_i$ has no term involving a function of $n_j$. Therefore, $w_i\indepe n_j$ holds. In the same way, $n_i\indepe w_j$ holds when there is no UCP from $x_i$ to $x_j$. When there is no UBP between $x_i$ and $x_j$, there is no external effect $n_k$ such that both $w_i$ and $w_j$ have terms involving functions of $n_k$. Therefore, $w_i\indepe w_j$ holds when there is no UBP between $x_i$ and $x_j$. Then, because there is no UCP or UBP between $x_i$ and $x_j$, $(w_i+n_i)\indepe (w_j+n_j)$ holds. From Equations~\ref{eq:proof3-1} and \ref{eq:proof3-2}, $(x_i-G_1(M))\indepe (x_j-G_2(N))$ holds. Hence, Equation~\ref{eq:lem3-2} is satisfied, and there is no contradictory when $x_j$ is a direct cause of $x_i$.

Next, we assume that $x_i$ is a direct cause of $x_j$. From this, we obtain

\begin{align}
	\begin{aligned}\label{eq:proof3-3}
		x_j=f_i^{(j)}(x_i)+ \sum_{x_m \in (P_j\setminus\{x_i\})} f_m^{(j)}(x_m)+w_j+n_j.
	\end{aligned}
\end{align}

As shown in Equation~\ref{eq:proof3-3}, $x_j$ has terms involving functions of $x_i$ when $x_i$ is a direct cause of $x_j$. In addition, $M$ cannot contain $x_i$, and $N$ cannot contain $x_i$ or $x_j$. Then, because of Lemma~\ref{lem:hosoku}, we obtain the following equations:
\begin{equation}
	\forall G_1, M\subseteq (X\setminus\{x_i\}), [(x_i - G_1(M)) \notindepe n_i].
\end{equation}

\begin{equation}
	\forall G_2, N\subseteq (X\setminus\{x_i,x_j\}), [(x_j - G_2(N)) \notindepe n_i].
\end{equation}
Then, because of Assumption~\ref{assumption:ass2}, we obtain 

\begin{equation}\label{eq:proof3-4}
	\forall G_1, G_2, M\subseteq (X\setminus\{x_i\}), N\subseteq (X\setminus\{x_i,x_j\}), [(x_i - G_1(M)) \notindepe (x_j - G_2(N))].
\end{equation}
This contradicts Equation~\ref{eq:lem3-2}.

Therefore, Equation~\ref{eq:lem3-2} is satisfied when $x_j$ is a direct cause of $x_i$, and it is not satisfied when $x_i$ is a direct cause of $x_j$. Hence, $x_j$ is a direct cause of $x_i$.

\end{proof}

\subsection{Proof of Lemma~4}
\begin{proof}
We prove the lemma by contradiction. Assume that there exists $x_j\in(K\setminus \{x_i\})$ that is a descendant of $x_i$, and that $(x_i-G_i^1(M_i\cup K\setminus \{x_i\}))$ and $(x_j-G_j^1(M_j))$ are mutually independent, and that $(x_i-G_i^2(M_i))$ and $(x_j-G_j^2(M_j))$ cannot be mutually independent. If $(x_i-G_i^2(M_i))$ and $(x_j-G_j^2(M_j))$ cannot be mutually independent, there exists a causal path from $x_i$ to $x_j$ that is not blocked by $M_j$. If such a causal path exists, $(x_i-G_i^1(M_i\cup K\setminus \{x_i\}))$ and $(x_j-G_j^1(M_j))$ cannot also be mutually independent. This contradicts the assumption defined above. Therefore, each $x_j\in K\setminus\{x_i\}$ is not a descendant of $x_i$.

\end{proof}

\section{Time complexity}
First, we obtain the time complexity required by Phase 1 of Algorithm~\ref{algo:algorithm}. The number of times lines 9--18 are repeated in Algorithm~\ref{algo:algorithm} is $_{p}C_t$ for each $t$. In addition, $t$ independence tests and GAM regressions are required in lines~10 and 13. Therefore, the number of independence tests and GAM regressions required in lines~10 and 13 is obtained by
\begin{align}
\begin{aligned}
\label{eq:time-1}
	\sum_{t=2}^{k=p} {}_{p}C_t \cdot t&=\sum_{t=2}^{t=p}\frac{p!}{t!(p-t)!}\cdot t\\
	&=p\cdot\sum_{t=2}^{t=p}\frac{(p-1)!}{(p-t)!(t-1)!}\\
	&=p(2^{p-1}-1).
\end{aligned}
\end{align}

The time complexity for calculating the HSIC is $\mathcal{O}(n^2)$~\citep{NIPS2007_3201}. Let $a$ denote the maximum number of splines for GAM regression. Then the time complexity for GAM regression is $\mathcal{O}(na^2)$~\citet{wood2004}. We assume $n$ is much larger than $a^2$. Then the time complexity of Phase 1 is obtained by 
\begin{equation}
\label{eq:time-2}
	\mathcal{O}\left(p(2^{p-1}-1)(n^2+na^2)\right)=\mathcal{O}\left( p2^pn^2 \right).
\end{equation}

Both Phase 2 in Algorithm~\ref{algo:algorithm} and Algorithm~\ref{algo:algorithm2} require $_{p}C_2$ calculations of the HSIC. Therefore, the time complexity is 
\begin{equation}
\label{eq:time-3}
	\mathcal{O}\left({}_pC_2n^2\right)=\mathcal{O}\left(p^2n^2\right),
\end{equation}
which is less than Equation~\ref{eq:time-2}. Therefore, the time complexity of our proposed method is $\displaystyle\mathcal{O}\left( p2^pn^2 \right)$.

\section{Average runtime of the proposed method}
The average runtime of the proposed method using artificial data (Section~\ref{sec:artifitial}) is shown in Figure~\ref{figure:time}. The details of the computing machine are as follows:
\begin{itemize}
	\item macOS Catalina 10.15.7
	\item 2.4 GHz 8‑core 9th‑generation Intel Core i9 processor
	\item 64 GB 2666 MHz DDR4 memory
	\item Python 3.8.6
\end{itemize}

\begin{figure}[h]
\centering
\includegraphics[width=7.0cm]{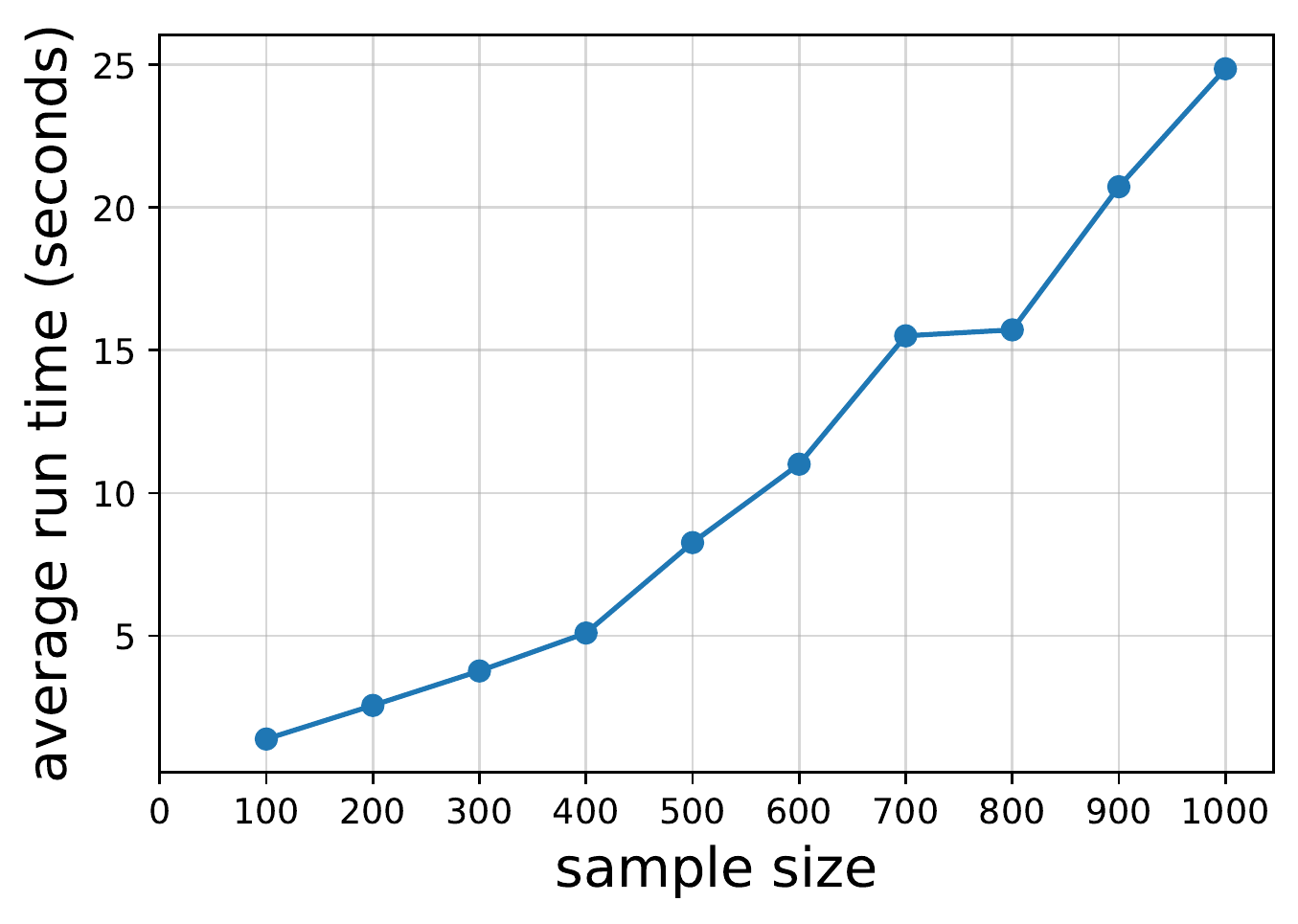}
\caption{Average runtime of the proposed method.}
\label{figure:time}
\end{figure}

\end{document}